\documentclass{article} % For LaTeX2e

\usepackage{times}
\usepackage{tikz}
\usepackage{graphicx}
\usepackage{algorithm}
\usepackage{algorithmic}
\usepackage{hyperref}
\usepackage{url}

\usepackage{achemso}
\setcitestyle{numbers,square}
\usepackage{nips15submit_e}

\usepackage{latexsym}
\usepackage{amsmath,amsfonts}
\usepackage{amssymb}
\usepackage{amsthm}
\usepackage{relsize}
\usepackage{mathtools}
\usepackage{tikz}
\usepackage{array}
\usepackage{subcaption}
\usepackage{comment}
\usepackage{multirow}
\usepackage{aliascnt}
\usepackage{xspace}
\usepackage[bb=fourier]{mathalfa}
\usepackage{siunitx}
\usepackage{tablefootnote}
\usepackage{multirow}
\usepackage{microtype}
\usepackage[export]{adjustbox}
\usepackage{footmisc}

\newcommand{\newSharedTheorem}[3]{
  \newaliascnt{#1}{#3}
  \newtheorem{#1}[#1]{#2}
  \aliascntresetthe{#1}
}

\newtheorem{thm}{Theorem}
\newSharedTheorem{lem}{Lemma}{thm}
\newtheorem*{obs}{Observation}
\newSharedTheorem{cor}{Corollary}{thm}
\newSharedTheorem{prop}{Proposition}{thm}
\theoremstyle{definition}
\newtheorem{defn}{Definition}
\theoremstyle{example}
\newtheorem*{expl}{Example}

\newcommand{\prob}[3]{p_{#3}(#1|#2)}

\newcommand{\reals}{{\mathbb{R}}}
\newcommand{\xs}{{\mathcal{X}}}
\newcommand{\ys}{{\mathcal{Y}}}
\newcommand{\normset}{{\mathcal{S}}}

\newcommand{\cexp}[1]{e^{#1}}

\newcommand{\ld}{\,\mathrm{d}}

\newcommand{\expect}{\mathbb{E}}

\newcommand{\cov}{\textrm{Cov}}

\newcommand{\mle}{{\hat{\eta}}}
\newcommand{\dmle}{{\hat{\eta}_\delta}}

\newcommand{\defeq}{\stackrel{\mathclap{\tiny\Delta}}{=}}

\newcommand{\hypercube}{\mathcal{H}}
\newcommand{\grad}{\nabla}

\newcommand{\mpunct}{\ }
\DeclareMathOperator*{\argmin}{arg\,min}
\DeclareMathOperator*{\argmax}{arg\,max}
\newcommand{\bigo}{\mathcal{O}}

\newcommand{\hypc}{\mathcal{H}}
\newcommand{\calS}{\mathcal{S}}
\newcommand{\calC}{\mathcal{C}}
\newcommand{\projC}{\Pi_{\calC}}
\newcommand{\LQR}{L^{2}\left(Q,~\R^{D}\right)}

\newcommand{\E}{\mathbb{E}}

\newcommand{\Var}{\mathrm{Var}}
\newcommand{\der}{\mathrm{d}}

\newcommand{\R}{\mathbb{R}}
\newcommand{\lspan}{\mathrm{span}}

\newcommand{\nml}{\left|\left|}
\newcommand{\nmr}{\right|\right|}
\newcommand{\VE}{V_{\mathrm{E}}}

\title{On the accuracy of self-normalized log-linear models}

\author{
  Jacob Andreas\thanks{Authors contributed equally.},~~~Maxim
  Rabinovich\footnotemark[1],~~~Dan Klein,~~Michael I.\ Jordan \\
Computer Science Division,
University of California, Berkeley \\
\texttt{\{jda,rabinovich,klein,jordan\}@cs.berkeley.edu} \\
}

\nipsfinalcopy % Uncomment for camera-ready version

\begin{document}

\maketitle

\begin{abstract}
  Calculation of the log-normalizer is a major computational obstacle in applications 
  of log-linear models with large output spaces. The problem of fast normalizer 
  computation has therefore attracted significant attention in the theoretical and
  applied machine learning literature. In this paper, we analyze a recently proposed
  technique known as ``self-normalization'', which introduces a regularization
  term in training to penalize log normalizers for deviating from zero.
  This makes it possible to use unnormalized model scores as approximate
  probabilities.
  Empirical evidence suggests that self-normalization is extremely effective,
  but a theoretical understanding of why it
  should work, and how generally it can be applied, is largely lacking.

  We prove generalization bounds on the estimated variance of normalizers
  and upper bounds on the loss in accuracy due to self-normalization, 
  describe classes of input distributions
  that self-normalize easily, and construct explicit examples of high-variance input
  distributions.
  Our theoretical results make predictions about the difficulty of fitting 
  self-normalized models to several classes of distributions, and we conclude
  with empirical validation of these predictions.
\end{abstract}

\section{Introduction}

Log-linear models, a general class that includes conditional random
fields (CRFs) and
generalized linear models (GLMs), offer a flexible yet tractable approach
modeling conditional probability distributions $p(x|y)$
\cite{Lafferty01CRF,McCullagh98GLM}. When the set of possible $y$ values is large, however,
the computational cost of computing a normalizing constant for each $x$ can be
prohibitive---involving a summation with many terms, a high-dimensional integral or an
expensive dynamic program.

The machine translation community has recently described several procedures for
training ``self-normalized'' log-linear models
\cite{Devlin2014NNJM,Vaswani13Neural}. The goal of self-normalization is to
choose model parameters that simultaneously yield accurate predictions and
produce normalizers clustered around unity. Model scores can then be used as
approximate surrogates for probabilities, obviating the computation normalizer
computation.

In particular, given a model of the form
\begin{equation}
  p_{\eta}(y~|~x) = e^{\eta^{T}T(y,~x) - A\left(\eta,~x\right)}
\end{equation}
with
\begin{equation}
  \label{eq:partition}
  A\left(\eta,~x\right) = \log \sum_{y \in \ys}{e^{\eta^{T}T(y,~x)}} \mpunct ,
\end{equation}
we seek a setting of $\eta$ such that $A(x, \eta)$ is
close enough to zero (with high probability under $p(x)$) to be ignored.

This paper aims to understand the theoretical properties of self-normalization.
Empirical results have already demonstrated the efficacy of this approach---for
discrete models with many output classes, it appears that normalizer values can
be made nearly constant without sacrificing too much predictive accuracy,
providing dramatic efficiency increases at minimal performance cost.

The broad applicability of self-normalization makes it likely to spread to other
large-scale applications of log-linear models, including structured prediction
(with combinatorially many output classes) and regression (with continuous
output spaces). But it is not obvious that we should expect such approaches to
be successful:
the number of inputs (if finite) can be on the order of millions, the geometry
of the resulting input vectors $x$ highly complex, and the class of functions
$A(\eta,~x)$ associated with different inputs quite rich. To find to find a
nontrivial parameter setting with $A(\eta,~x)$ roughly constant seems
challenging enough; to require that the corresponding $\eta$ also lead to good
classification results seems too much. And yet for
many input distributions that arise in practice, it appears possible to choose $\eta$
to make $A(\eta,~x)$ nearly constant without having to sacrifice classification
accuracy.  %elf-normalizing procedures, and normalizable training distributions.

Our goal is to bridge the gap between theoretical intuition and
practical experience. Previous work \cite{Andreas15SelfNorm}
bounds the sample complexity of self-normalizing training procedures
for a restricted class of models, but leaves open the question of how
self-normalization interacts with the predictive power of the learned
model.
This paper seeks to answer that question. We begin by generalizing
the previously-studied model to a much more general class
of distributions, including distributions with
continuous support (\autoref{sec:preliminaries}). Next, we
provide what we believe to be the first characterization of the interaction
between self-normalization and model accuracy \autoref{sec:lgap}. This
characterization is given from two perspectives:

\begin{itemize}

\item
a bound on the ``likelihood gap'' between self-normalized and unconstrained
models

\item
a conditional distribution provably hard to represent with a self-normalized
model
\end{itemize}

In \autoref{sec:experiments}, we present empirical evidence that these bounds
correctly characterize the difficulty of self-normalization, and in the
conclusion we survey a set of open problems that we believe merit further
investigation.

\section{Problem background}

%The technique introduced in Devlin XXX achieves this via explicit penalization of
%the normalizer magnitudes. In the language of log-linear models, the output layer 
%of the neural network can be interpreted as a model with inputs $x \in \xs$ and discrete outputs $y \in \ys$, satisfying
%$$ p_{\eta}(y~|~x) = \exp\left(\eta^{T}T(y,~x) - A\left(\eta,~x\right)\right) $$
%and
%$$ A\left(\eta,~x\right) = \sum_{y \in \ys}{e^{\eta^{T}T(y,~x)}} . $$
%The explicitly penalized training objective then takes the form
%\begin{align}\label{eq:devlin}
%\LDev(\eta) & = \sum_{i}\left[\eta^{T}T(y_{i},~x_{i}) - A\left(\eta,~x_i\right)\right]  \\
%                 &~~~~~~~~~~ - \alpha \sum_i \left(e^{A\left(\eta,~x_i\right)}\right)^{2} \nonumber ,
%\end{align}
%where we have denoted the union of the parameters of all the output units as $\eta$ and let $T(y,~x_i)$ denote the input to output unit $y$ when the
%network is fed with data point $i$. 

The immediate motivation for this work is a procedure proposed to
speed up decoding in a machine translation system with a neural-network language
model \cite{Devlin2014NNJM}. The language model used is a standard feed-forward neural network, with a
``softmax'' output layer that turns the network's predictions into a
distribution over the vocabulary, where each probability is log-proportional to
its output activation. It is observed                                                                                    that with a sufficiently large
vocabulary, it becomes prohibitive to obtain probabilities from this model 
(which must be queried millions of times during decoding). To fix this, the
language model is trained with the following objective:
\begin{align}
  \label{eq:devlin}
  \max_W \sum_i \Big[ &N(y_i|x_i;W) - \log \sum_{y'} \cexp{N(y'|x_i;W)} \nonumber
  - \alpha \Big( \log \sum_{y'} \cexp{N(y'|x_i;W)}\Big)^2 \Big]
\end{align}
where $N(y|x;W)$ is the response of output $y$ in the neural net with weights
$W$ given an input $x$. 
From a Lagrangian perspective, the extra penalty term simply confines the
$W$ to the set of ``empirically normalizing'' parameters, for which all log-normalizers are close 
(in squared error) to the origin.
For a suitable choice of $\alpha$, it is observed that the
trained network is simultaneously accurate enough to produce good translations,
and close enough to self-normalized that the raw scores $N(y_i | x_i)$ can be used in
place of log-probabilities without substantial further degradation in quality.

%As noted in the introduction, this is quite surprising---it is not obvious that
%there should exist nontrivial ``self-normalizing'' weights settings of this kind
%that generalize reliably to predictions on novel inputs, and we seek to
%understand the observed behavior. 
We seek to understand the observed success of these models in finding accurate,
normalizing parameter settings.  While it is possible to derive
bounds of the kind we are interested in for general neural networks
\cite{Bartlett98NNSample}, in this paper we work with a simpler  linear 
parameterization that we believe captures the interesting aspects of this problem. 
\footnote{It is possible to view a log-linear model as a single-layer network with a
  softmax output. More usefully, all of the results presented here apply
  directly to trained neural nets in which the last layer only is \emph{retrained}
  to self-normalize \cite{Anthony09NNTF}.}

\subsection*{Related work}

The approach described at the beginning of this section is closely related to an alternative
self-normalization trick described based on noise-contrastive estimation (NCE) \cite{Gutmann10NCE}. 
NCE is an alternative to
direct optimization of likelihood, instead training a classifier to distinguish
between true samples from the model, and ``noise'' samples from some other
distribution.  The structure of the training objective makes it possible to
replace explicit computation of each log-normalizer with an estimate. In
traditional NCE, these values are treated as part of the parameter space, and
estimated simultaneously with the model parameters;
%\citeauthor{Gutmann10NCE} provide 
there exist
guarantees that the normalizer estimates will eventually converge to
their true values. It is instead possible to fix all of these estimates to one.
In this case, empirical evidence suggests that the resulting model will also
exhibit self-normalizing behavior \cite{Vaswani13Neural}.

%Devlin XXX, by contrast, explicitly penalize large partitions by adding an extra
%``regularization'' term to the likelihood objective of the form
%\[ \alpha \sum_i A(x_i, \eta)^2 \]
%where the parameter $\alpha$ controls the relative importance of the
%normalization penalty and the standard likelihood term. Unlike Vaswani XXX, this
%approach does require explicit computation of the log-partition for all training
%examples at training time, but also eliminates it at test time, and makes it
%possible for users to explicitly trade off between the relative importance of
%likelihood and normalization. Like Vaswani XXX, they observe that the resulting
%models generally have small log-partitions even in novel contexts.

A host of other techniques exist for solving the computational problem
posed by the log-normalizer. Many of these involve approximating the associated
sum or integral using quadrature \cite{OHagan91Quadrature}, herding
\cite{Chen12KernelHerding}, or Monte Carlo methods \cite{Doucet10SMC}.
For the special case of discrete, finite output spaces, an alternative
approach---the hierarchical softmax---is to replace the large sum in the
normalizer with a series of binary decisions
\cite{Morin05HierarchicalSoftmax}. The output classes are arranged in a binary
tree, and the probability of generating a particular output is the product of
probabilities along the edges leading to it. This reduces the cost of computing
the normalizer from $\bigo(k)$ to $\bigo(\log k)$. While this limits the set of
distributions that can be learned, and still requires greater-than-constant
time to compute normalizers, it appears to work well in practice. It cannot,
however, be applied to problems with continuous output spaces.

% The theoretical properties of numerical approximations to the partition are well
% understood, but the consequences of constrained partitions much less so. 
% We
% XXX previously investigated some properties of approximate normalization for the
% special case of log-linear models with finite output spaces and
% class-independent features; here we extend those results to exponential families
% in general, provide new characterizations of approximately normalizable
% distributions, and study in greater detail the loss of predictive power that
% results from normalization constraints.

\section{Self-normalizable distributions}
\label{sec:preliminaries}
\label{sec:norm-dist}

We begin by providing a slightly more formal characterization of a general
log-linear model:
\begin{defn}[\bf Log-linear models]
  Given a space of \emph{inputs} $\xs$, a space of \emph{outputs} $\ys$, a
  measure $\mu$ on $\ys$, a nonnegative function $h : \ys \to \reals$, and a
  function $T : \xs \times \ys \to \reals^d$ that is $\mu$-measurable with
  respect to its second argument, we can define a \emph{log-linear model}
  indexed by parameters $\eta \in \reals^d$, with the form
  \begin{equation}
    \prob{y}{x}{\eta} = h(y) \cexp{\eta^\top T(x,y) - A(x,\eta)} \mpunct ,
  \end{equation}
  where
  \begin{equation}
    \label{eq:log-partition}
    A(x,\eta) \defeq \log \int_\ys h(y) \cexp{\eta^\top T(x,y)} \ld \mu(y) \mpunct .
  \end{equation}
  If $A(x,\eta) \leq \infty$, then $\int_y \prob{y}{x}{\eta} \ld \mu(y) = 1$,
  and $\prob{y}{x}{\eta}$ is a probability density over $\ys$.\footnote{Some
    readers may be more familiar with generalized linear models, which also
    describe exponential family distributions with a linear dependence on input.
    The presentation here is strictly more general, and has a few notational
    advantages: it makes explicit the dependence of $A$ on $x$ and $\eta$ but
    not $y$, and lets us avoid tedious bookkeeping involving natural and mean
    parameterizations. \cite{Yang12GLM}}
 % TODO no longer necessary?
\end{defn}

We next formalize our notion of a self-normalized model.

\begin{defn}[\bf Self-normalized models]
  The log-linear model $\prob{y}{x}{\eta}$ is
  \emph{self-normalized with respect to} a set $\normset \subset \xs$ if for all $x \in
  \normset$, $A(x,\eta) = 0$. In this case we say that $\normset$ is
  \emph{self-normalizable}, and $\eta$ is \emph{self-normalizing} w.r.t.\ $\normset$.
\end{defn}

An example of a normalizable set is shown in \autoref{fig:normalizables}, and we
provide additional examples below:

\begin{figure}
\begin{subfigure}{0.48\textwidth}
  \center
  \vspace{2.8em}
  \includegraphics[width=\columnwidth]{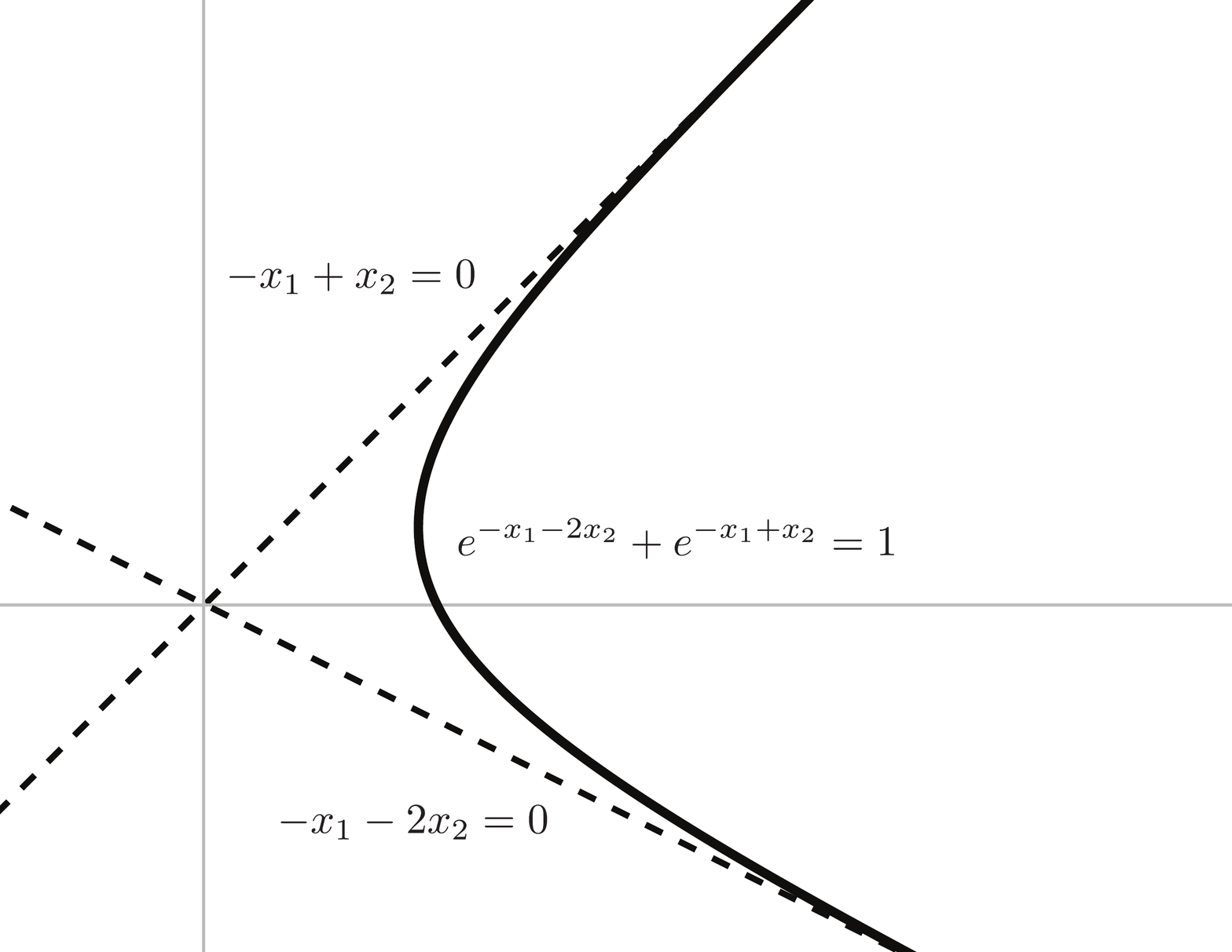}
  \caption{A \emph{self-normalizable} set $\normset$ for fixed $\eta$: 
  the solutions $(x_1, x_2)$ to $A(x, \eta) = 0$ with $\eta^\top
  T(x,y) = \eta_y^\top x$ and $\eta = \{ (-1, 1), (-1, -2) \}$. The set forms a
  smooth one-dimensional manifold bounded on either side by hyperplanes normal
to $(-1, 1)$ and $(-1, -2)$.}
  \label{fig:normalizables}
\end{subfigure}
\hfill
\begin{subfigure}{0.48\textwidth}
  \center
  \includegraphics[width=\columnwidth]{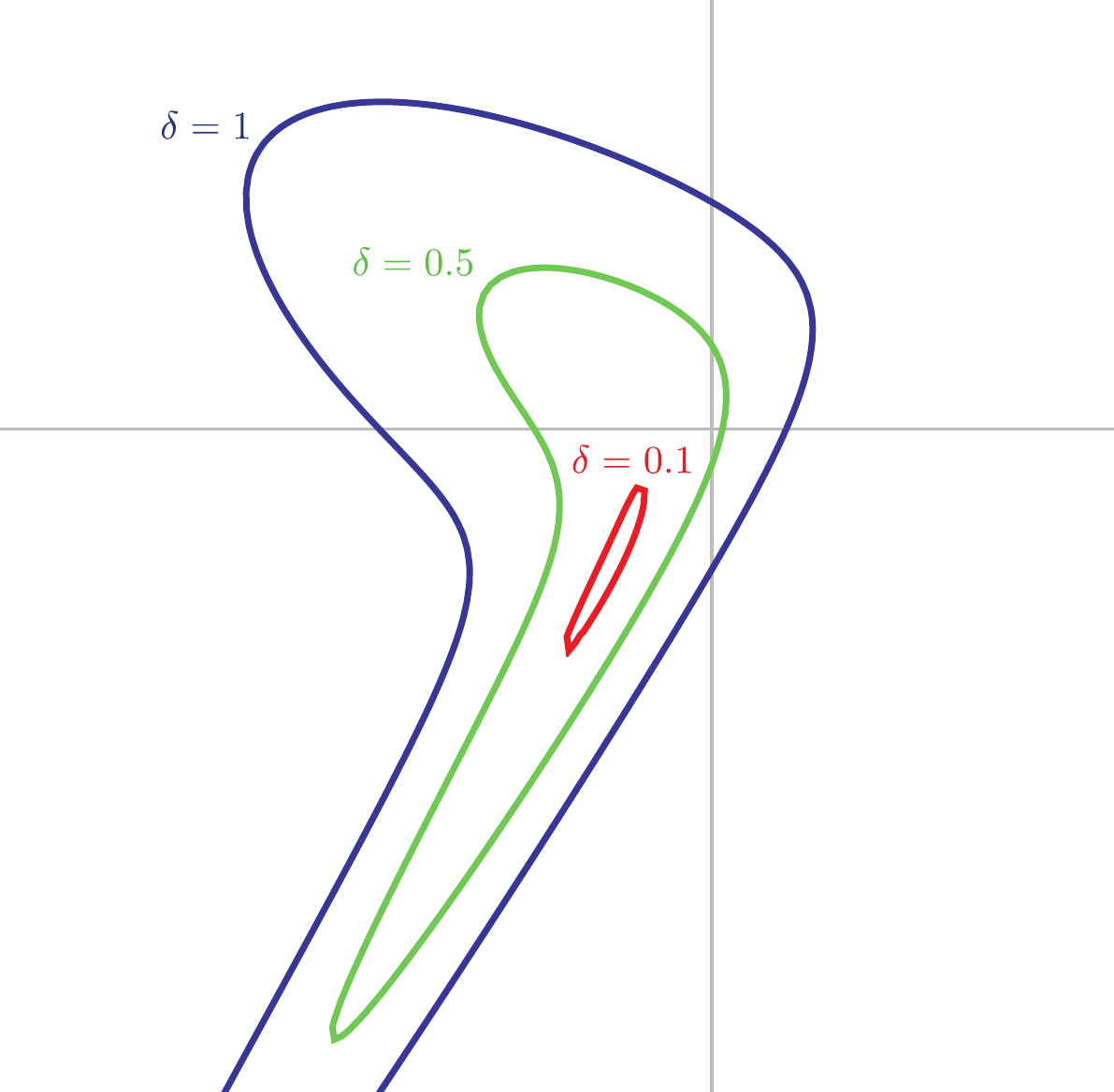}
  \caption{Sets of \emph{approximately normalizing} parameters $\eta$ for
  fixed $p(x)$: solutions $(\eta_1, \eta_2)$ to $\expect[ A(x, \eta)^2 ] =
  \delta^2$ with $T(x, y) = (x+y, -xy)$, $y \in \{-1, 1\}$ and $p(x)$ uniform on $\{1,
  2\}$. For a given upper bound on normalizer variance, the feasible set of
  parameters is nonconvex, and grows as $\delta$ increases.}
  \label{fig:normalizing}
\end{subfigure}
\caption{Self-normalizable data distributions and parameter sets.}
\end{figure}

    \pagebreak
\begin{expl}
  Suppose 
  \vspace{-.5em}
  \begin{align*}
    \normset &= \{ \log 2, -\log 2 \} \mpunct , \\
    \ys &= \{-1, 1\} \\
    T(x, y) &= [xy, 1] \\
    \eta &= (1, \log(2/5)) \mpunct . \\
    \intertext{Then for either $x \in \normset$,}
    A(x, \eta) &= \log(\cexp{\log 2 + \log(2/5)} + \cexp{-\log 2 + \log(2/5)})
    \\
    &= \log((2/5)(2 + 1/2)) \\
    &= 0 \mpunct, 
  \end{align*}
  and $\eta$ is self-normalizing with respect to $\normset$.
\end{expl}

It is also easy to choose parameters that do not result in a self-normalized distribution,
and in fact to construct a target distribution which cannot be self-normalized:

\begin{expl}
  Suppose
  \begin{align*}
    \xs &= \{(1, 0), (0, 1), (1, 1)\}  \\
    \ys &= \{-1, 1\} \\
    T(x, y) &= (x_1 y, x_2 y, 1)
  \end{align*}
  Then there is no $\eta$ such that $A(x, \eta) = 0$ for all $x$,
  and $A(x, \eta)$ is constant if and only if $\eta = \mathbf{0}$.
\end{expl}

As previously motivated, downstream uses of these models may be robust to small
errors resulting from improper normalization, so it would be useful to
generalize this definition of normalizable distributions to distributions that
are only approximately normalizable.  Exact normalizability of the conditional
distribution is a deterministic statement---there either does or does not exist
some $x$ that violates the constraint. In \autoref{fig:normalizables}, for
example, it suffices to have a single $x$ off of the indicated surface to make a
set non-normalizable.  Approximate normalizability, by contrast, is inherently a
\emph{probabilistic} statement, involving a distribution $p(x)$ over inputs.
Note carefully that we are attempting to represent $p(y|x)$ but have no
representation of (or control over) $p(x)$, and that approximate normalizability
depends on $p(x)$ but not $p(y|x)$.

Informally, if some input violates the self-normalization constraint by a large
margin, but occurs only very infrequently, there is no problem; instead we are
concerned with \emph{expected} deviation. It is also at this stage that the
distinction between penalization of the normalizer vs.\ log-normalizer becomes
important. The normalizer is necessarily bounded below by zero (so overestimates
might appear much worse than underestimates), while the log-normalizer is
unbounded in both directions. For most applications we are concerned with log
probabilities and log-odds ratios, for which an expected normalizer close to
zero is just as bad as one close to infinity. Thus the log-normalizer is the
natural choice of quantity to penalize.

% definition of an approximately normalized model

\begin{defn}[\bf Approximately self-normalized models]
  The log-linear distribution $\prob{y}{x}{\eta}$ is
  \emph{$\delta$-approximately normalized with respect to} a distribution $p(x)$ over
    $\xs$ if $\expect[A(X,\eta)^2] < \delta^2$. In this case we say that $p(x)$ is
    \emph{$\delta$-approximately self-normalizable}, and $\eta$ is
    \emph{$\delta$-approximately self-normalizing}.
\end{defn}

The sets of $\delta$-approximately self-normalizing parameters for a fixed input
distribution and feature function are depicted in \autoref{fig:normalizing}.
Unlike self-normalizable sets of \emph{inputs}, self-normalizing and
approximately self-normalizing sets of \emph{parameters} may have complex geometry.

%At this point we observe that $\expect[A(X,\eta)^2]$ would simply be the
%variance of $A(x)$ if $\expect A(X,\eta)$ were known to be zero. In fact, given any log-linear model with
%sufficient statistics $T$ and parameters $\eta$, it is always possible to
%construct an identical distribution with $\expect A = 0$ (or any other
%value)---simply take $T'(x, y) = [T(x, y),\  1]$ and $\eta' = [\eta,\
%-\expect A]$.
% TODO concatenation direction
%Thus in the following analysis we may choose the mean to be whatever is
%convenient almost without loss of generality; normalization errors resulting
%from uncertainty about $\expect A$ itself vanish asymptotically and can be dealt
%with straightforwardly using the central limit theorem. For the sake of
%simplicity, we will assume that the mean normalizer in expectation can
%be chosen exactly from only finite samples.

Throughout this paper, we will assume that vectors of sufficient statistics
$T(x,y)$ have bounded $\ell_2$ norm at most $R$, natural parameter vectors
$\eta$ have $\ell_2$ norm at most $B$ (that is, they are Ivanov-regularized),
and that vectors of both kinds lie in $\reals^d$. Finally,
we assume that all input vectors have a constant feature---in particular, that
$x_0 = 1$ for every $x$ (with corresponding weight $\eta_0$).
\footnote{It will occasionally be instructive to consider the special
case where $\xs$ is the Boolean hypercube, and we will explicitly note where
this assumption is made. Otherwise all results apply to general distributions,
both continuous and discrete.}

%\section{Self-normalizable distributions}

The first question we must answer is whether the problem of training self-normalized
models is feasible at all---that is, whether there exist any exactly
self-normalizable data distributions $p(x)$, or at least $\delta$-approximately
self-normalizable distributions for small $\delta$.  \autoref{sec:preliminaries}
already gave an example of an exactly normalizable distribution. In fact, there
are large classes of both exactly and approximately normalizable distributions.

\begin{obs}
  Given some fixed $\eta$, consider the set $S_\eta = \{ x \in \xs : A(x,\eta) =
  0 \}$. Any distribution $p(x)$ supported on $S_\eta$ is
  normalizable. Additionally, every self-normalizable distribution is characterized
  by at least one such $\eta$.
\end{obs}

This definition provides a simple geometric
characterization of self-normalizable distributions. An example solution set is shown
in \autoref{fig:normalizables}. More generally, if $y$ is discrete and $T(x,y)$
consists of $\left|\ys\right|$ repetitions of a fixed feature function $t(x)$ (as in \autoref{fig:normalizables}), then we can write
\begin{equation}
  A(x, \eta) = \log \sum_{y \in \ys} \cexp{\eta_y^\top t(x)} .
\end{equation}
Provided $\eta_y^\top t(x)$ is convex in $x$ for each $\eta_y$, the level sets of
$A$ as a function of $x$ form the boundaries of convex sets. In particular, exactly
normalizable sets are always the boundaries of convex regions, as in the simple example
\autoref{fig:normalizables}.

% TODO I'm pretty sure that in the continuous case you can approximate any convex
% set. Worth doing?

We do not, in general, expect real-world datasets to be supported on the precise
class of self-normalizable surfaces. Nevertheless, it is very often observed
that data of practical interest lie on other low-dimensional manifolds within
their embedding feature spaces. Thus we can ask whether it is sufficient for a
target distribution to be well-approximated by a self-normalizing one. We begin
by constructing an appropriate measurement of the quality of this approximation.

%It is less obvious how to usefully characterize the approximately self-normalizable
%distributions. One way to make progress is to observe that real-world data
%distributions often lie on low-dimensional manifolds within their embedding feature 
%spaces. In this case, we can obtain a bound on expected deviation in the normalizer 
%from zero. 

\begin{defn}[\bf Closeness]
  An input distribution $p(x)$ is \emph{$D$-close} to a set $\normset$ if
  \begin{equation}
    \expect \left[ \inf_{x^* \in \normset} \sup_{y \in \ys} ||T(X,y) - T(x^*,y)||_2
            \right] \leq D
  \end{equation}
\end{defn}

In other words, $p(x)$ is $D$-close to $\normset$ if a random sample from $p$ is
no more than a distance $D$ from $\normset$ in expectation.
Now we can relate the quality of this approximation to the level of
self-normalization achieved. Generalizing a result from
\cite{Andreas15SelfNorm}, we have:

\begin{prop}
  \label{prop:approx-closeness}
  Suppose $p(x)$ is $D$-close to $\{ x : A(x,\eta) = 1 \}$.
 % , and $||\eta||_2 \leq B$. 
  Then $p(x)$ is $BD$-approximately self-normalizable (recalling that $||x||_2 \leq
  B$).
\end{prop}

(Proofs for this section may be found in \autoref{app:norm-dist}.)

The intuition here is that data distributions that place most of their mass in
feature space close to normalizable sets are approximately normalizable on the
same scale.

\section{Normalization and model accuracy}
\label{sec:lgap}
\label{sec:var-lower-bd}

So far our discussion has concerned the problem of finding conditional
distributions that self-normalize, without any concern for how well they actually
perform at modeling the data. Here the relationship between the approximately
self-normalized distribution and the true distribution $p(y|x)$ (which we have so far
ignored) is essential. Indeed, if we are not concerned with making a good model
it is always trivial to make a normalized one---simply take $\eta = \mathbf{0}$
and then scale $\eta_0$ appropriately!  We ultimately desire both good
self-normalization and good data likelihood, and in this section we characterize
the tradeoff between maximizing data likelihood and satisfying a
self-normalization constraint.

We achieve this characterization by measuring the \emph{likelihood gap} between
the classical maximum likelihood estimator, and the MLE subject to a
self-normalization constraint. Specifically, given pairs $((x_1, y_1),
(x_2, y_2), \dots, (x_n, y_n))$, let $\ell(\eta|x,y) = \sum_i \log
\prob{y_i}{x_i}{\eta}$. Then define
\begin{align}
  \mle &= \argmax_\eta \ell(\eta|x,y) \\
  \dmle &= \argmax_{\eta : V(\eta) \leq \delta} \ell(\eta|x,y)
\end{align}
(where $V(\eta) = \frac{1}{n} \sum_i A(x_i, \eta)^2$).

We would like to obtain a bound on the \emph{likelihood gap}, which we define as the quantity 
\begin{equation}
  \Delta_\ell(\mle, \dmle) = \frac{1}{n} (\ell(\mle|x,y) - \ell(\dmle|x,y)) \mpunct .
\end{equation}
We claim:
\begin{thm}
  \label{prop:lgap}
  %Let $||T(x,y)||_2 \leq R$ for all $x$ and $y$, and
  Suppose $\ys$ has finite measure.
  Then asymptotically as $n \rightarrow \infty$
  \begin{equation}
    \hspace{-.4em}
    \Delta_\ell(\mle, \dmle) \leq \left(1 - \frac{\delta}{R||\mle||_2}\right)
    \expect\, \mathrm{KL}(\prob{\cdot}{X}{\eta}\ ||\ \mathrm{Unif})
    \mpunct .
  \end{equation}
\end{thm}

(Proofs for this section may be found in \autoref{app:lgap}.)

% TODO moving the expectation
This result lower-bounds the likelihood at $\dmle$ by
explicitly constructing a scaled version of $\mle$ that satisfies the
self-normalization constraint. Specifically, if $\eta$ is chosen so that normalizers
are penalized for distance from $\log \mu(\ys)$ (e.g.\ the logarithm of the
number of classes in the finite case), then any increase in $\eta$ along the
span of the data is guaranteed to increase the penalty. From here it is possible
to choose an $\alpha \in (0, 1)$ such that $\alpha\mle$ satisfies the
constraint. The likelihood at $\alpha\mle$ is necessarily less than
$\ell(\dmle|x,y)$, and can be used to obtain the desired lower bound.

Thus at one extreme, distributions close to uniform can be self-normalized with
little loss of likelihood. What about the other extreme---distributions ``as far
from uniform as possible''? With suitable
assumptions about the form of $\prob{y}{x}{\hat{\eta}}$, we can use the same
construction of a self-normalizing parameter to achieve an alternative
characterization for distributions that are close to deterministic:
\begin{prop}
  \label{prop:strong-lgap}
  Suppose that $\xs$ is a subset of the Boolean hypercube, $\ys$ is finite, and
  $T(x,y)$ is the conjunction of each element of $x$ with an indicator on the
  output class.  Suppose additionally that in \emph{every} input $x$,
  $\prob{y}{x}{\mle}$ makes a unique best prediction---that is, for each $x \in \xs$,
  there exists a unique $y^{\ast} \in \ys$ such that whenever $y \neq y^{\ast}$,
  $\eta^\top T(x,y^{\ast}) > \eta^\top T(x,y)$. Then
  %$\forall x.
  %\exists y^*. (y \neq y^*)
  %\rightarrow
  %\mle^\top T(x,y^*) > \mle^\top T(x,y)$. Then
  \begin{equation}
    \Delta_\ell(\mle, \dmle) \leq b \left(||\eta||_2 - \frac{\delta}{R}\right)^2
    \cexp{-c \delta / R}
  \end{equation}
  for distribution-dependent constants $b$ and $c$.
\end{prop}
This result is obtained by representing the constrained likelihood with a
second-order Taylor expansion about the true MLE. All terms in the likelihood
gap vanish except for the remainder; this can be upper-bounded by the
$||\dmle||_2^2$ times the largest eigenvalue the feature covariance matrix at
$\dmle$, which in turn is bounded by $\cexp{-c\delta/R}$. 

The favorable rate we obtain for this case indicates that ``all-nonuniform''
distributions are also an easy class for self-normalization. Together with
\autoref{prop:lgap}, this suggests that hard distributions must have some mixture
of uniform and nonuniform predictions for different inputs. This is supported by
the results in \autoref{sec:var-lower-bd}.

%\section{High-variance distributions}

The next question is whether there is a corresponding lower bound; that is,
whether there exist any conditional distributions for which all nearby
distributions are provably hard to self-normalize. The existence of a direct analog
of \autoref{prop:lgap} remains an open problem, but we make progress by
developing a general framework for analyzing normalizer variance.

One key issue is that while likelihoods are invariant to certain changes in the
natural parameters, the log normalizers (and therefore their variance) is far
from invariant. We therefore focus on equivalence classes of natural parameters,
as defined below. Throughout, we will assume a fixed
distribution $p(x)$ on the inputs $x$. 

\begin{defn}[\bf Equivalence of parameterizations]
Two natural parameter values $\eta$ and $\eta'$ are said to be {\it equivalent} (with respect to an input distribution $p(x)$), denoted $\eta \sim \eta'$ if 
$$ p_\eta(y|X) = p_{\eta'}(y|X) \quad \text{a.s. } p(x) $$

\end{defn}

We can then define the optimal log normalizer variance for the distribution associated with a natural parameter value. 

\begin{defn}[\bf Optimal variance]
We define the {\it optimal log normalizer variance} of the log-linear model associated with a natural parameter value $\eta$ by
$$ V^{\ast}(\eta) = \inf_{\eta' \sim \eta} \Var_{p(x)}\left[A(X,~\eta)\right] . $$
\end{defn}

We now specialize to the case where $\ys$ is finite with $\left|\ys\right| = K$ and where $T \colon \ys \times \xs \rightarrow \R^{Kd}$ satisfies
$$ T(k,~x)_{k'j} = \delta_{kk'}x_{j} . $$
This is an important special case that arises, for example, in multi-way logistic regression. In this setting, we can show that despite the
fundamental non-identifiability of the model, the variance can still be shown to be high under \emph{any} parameterization of the distribution.

\begin{thm}\label{thm:norm-var-lo}
Let $\xs = \lbrace 0,~ 1 \rbrace^{d}$ and let the input distribution $p(x)$ be uniform on $\xs$. There exists an $\eta^{0} \in \R^{Kd}$ such that for $\eta = \alpha\eta^{0}$, $\alpha > 0$,
$$ V^{\ast}(\eta) \geq \frac{\left|\left|\eta\right|\right|_{2}^{2}}{32d(d - 1)} - 4Ke^{-\frac{\sqrt{1 - \frac{1}{d}}\left|\left|\eta\right|\right|_{2}}{2(d - 1)}}\left|\left|\eta\right|\right|_{2} . $$ 
\end{thm}

%Proofs for this section may be found in \autoref{app:var-lower-bd}.

\section{Experiments}

The high-level intuition behind the results in the preceding section can be
summarized as follows: 1) for predictive distributions that are in expectation
high-entropy or low-entropy, self-normalization results in a relatively small
likelihood gap; 2) for mixtures of high- and low-entropy distributions,
self-normalization may result in a large likelihood gap.  More generally, we
expect that an increased tolerance for normalizer variance will be associated
with a decreased likelihood gap.

In this section we provide experimental confirmation of these predictions. We
begin by generating a set of random sparse feature vectors, and an initial
weight vector $\eta_0$.  In order to produce a sequence of label distributions
that smoothly interpolate between low-entropy and high-entropy, we introduce a
temperature parameter $\tau$, and for various settings of $\tau$ draw labels
from $p_{\tau \eta}$. We then fit a self-normalized model to these training
pairs. In addition to the synthetic data, we compare our results to empirical
data \cite{Devlin2014NNJM} from a self-normalized language model.

\autoref{fig:tradeoff} plots the tradeoff between the likelihood gap and the
error in the normalizer, under various distributions (characterized by their KL
from uniform). Here the tradeoff between self-normalization and model accuracy
can be seen---as the normalization constraint is relaxed, the likelihood gap
decreases.

\autoref{fig:fixed_delta} shows how the likelihood gap varies as a function of
the quantity $\expect KL(p_\eta(\cdot|X)||\textrm{Unif})$. As predicted, it can
be seen that both extremes of this quantity result in small likelihood gaps,
while intermediate values result in large likelihood gaps.

\begin{figure}
\begin{subfigure}{0.48\textwidth}
  \center
  {
  \footnotesize
  $\Delta_\ell$
  \includegraphics[width=.929\columnwidth,valign=c]{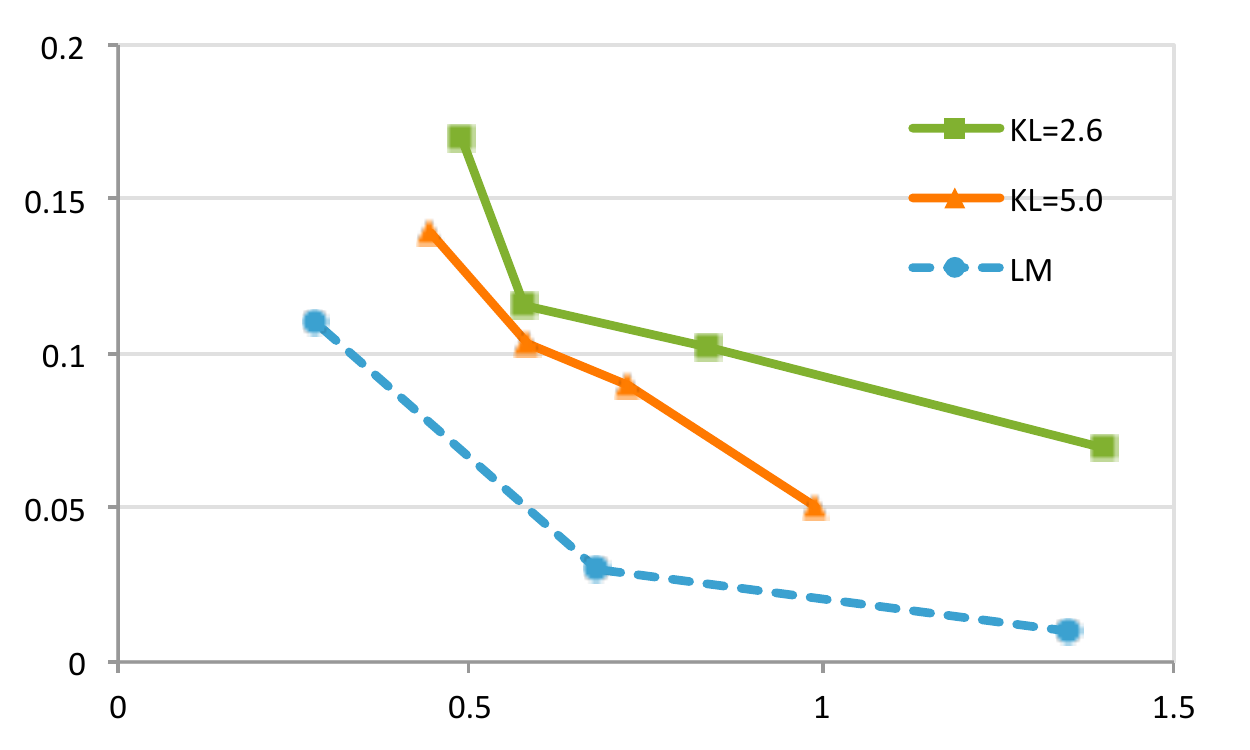} \\
  $\qquad\quad \delta$
  }
  \caption{Normalization / likelihood tradeoff. As the normalization
  constraint $\delta$ is relaxed, the likelihood gap $\Delta_\ell$ decreases.
  Lines marked ``KL='' are from synthetic data; the line marked ``LM'' is from
  \cite{Devlin2014NNJM}.}
  \label{fig:tradeoff}
\end{subfigure}
\hfill
\begin{subfigure}{0.48\textwidth}
  \center
  {
  \footnotesize
  $\Delta_\ell$
  \includegraphics[width=.9\columnwidth,valign=c]{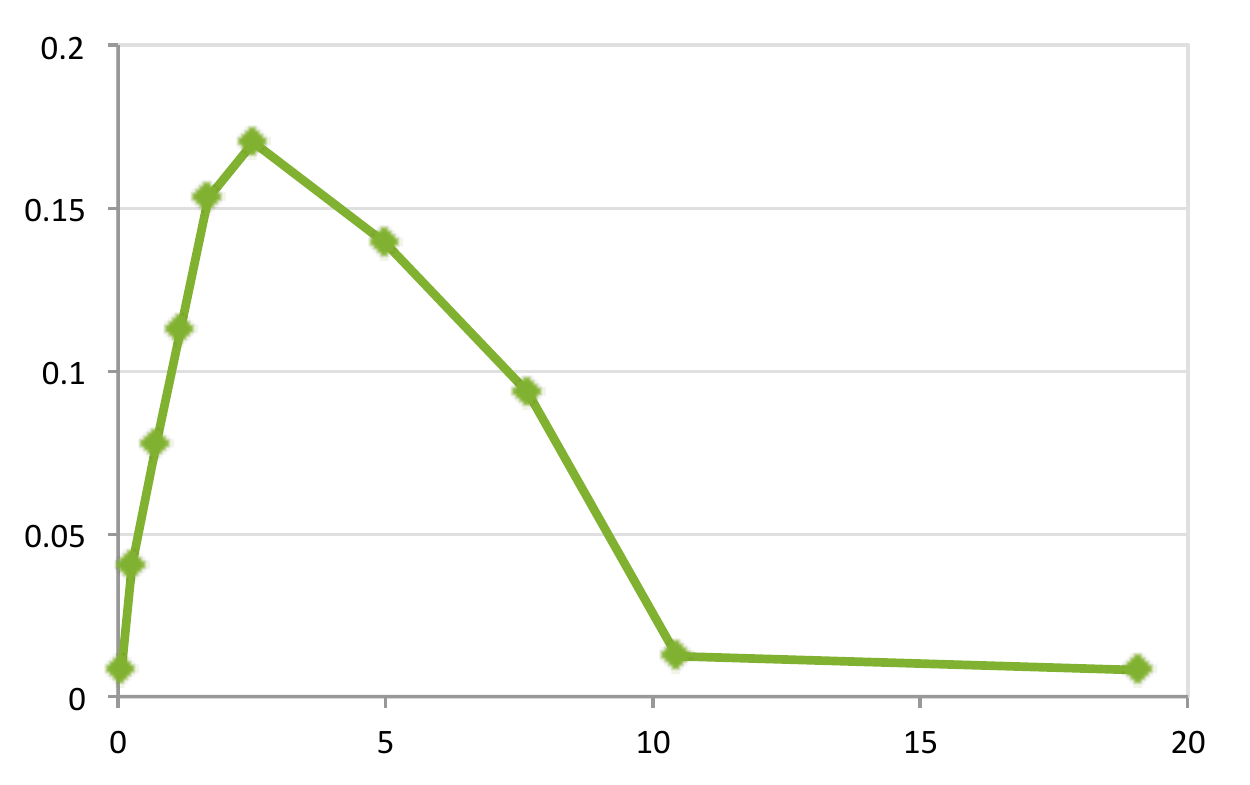} \\
  $\qquad\quad \expect KL(p_\eta||\textrm{Unif})$
  }
  \caption{Likelihood gap as a function of expected divergence from the uniform
  distribution.  As predicted by theory, the likelihood gap increases, then
  decreases, as predictive distributions become more peaked.}
  \label{fig:fixed_delta}
\end{subfigure}
\caption{Experimental results}
\end{figure}

% The high-level intuition behind the results in the preceding sections can be summarized as follows: when self-normalized,
% \begin{itemize}
% \item All distributions that are everywhere close to uniform suffer a small likelihood gap
% \item Some distributions that are everywhere close to deterministic suffer a small likelihood gap
% \item There exists a distribution neither everywhere uniform nor everywhere deterministic that is hard to normalize XXX precise
% \end{itemize}
% 
% Thus for a fixed delta, we expect our distributions to look something like this:
% 
% [the sketch]
% 
% They do!
% 
% [empirical curves]

\label{sec:experiments}

\section{Conclusions}

Motivated by the empirical success of self-normalizing parameter estimation
procedures for log-linear models, we have attempted to establish a theoretical
basis for the understanding of such procedures. We have characterized both
self-normalizable \emph{distributions}, by constructing provably easy examples,
and self-normalizing \emph{training procedures}, by bounding the loss of
likelihood associated with self-normalization.

While we have addressed many of the important first-line theoretical questions
around self-normalization, this study of the problem is by no means complete. We
hope this family of problems will attract further study in the larger machine
learning community; toward that end, we provide the following list of open
questions:

\begin{enumerate}
  \item {\bf How else can the approximately self-normalizable distributions be
    characterized?} The class of approximately normalizable distributions we
    have described is unlikely to correspond perfectly to real-world data. We
    expect that \autoref{prop:approx-closeness} can be generalized to other
    parametric classes, and relaxed to accommodate spectral or sparsity
    conditions.

  \item {\bf Are the upper bounds in \autoref{prop:lgap} or
    \autoref{prop:strong-lgap} tight?} 
    %So far we have been unable to construct
    %explicit examples of distributions that achieve worst-case likelihood gaps.
    Our constructions involve relating the normalization constraint to the
    $\ell_2$ norm of $\eta$, but in general some parameters can have very large
    norm and still give rise to almost-normalized distributions. 
    %A proof not relying on parameter norms might achieve tighter bounds.

  \item {\bf Do corresponding lower bounds exist?} While it is easy to construct
    of exactly self-normalizable distributions (which suffer no loss
    of likelihood), we have empirical evidence that hard distributions also
    exist. It would be
    useful to lower-bound the loss of likelihood in terms of some simple
    property of the target distribution.

  \item {\bf Is the hard distribution in \autoref{thm:norm-var-lo} stable?} This
    is related to the previous question. The existence of high-variance
    distributions is less worrisome if such distributions are comparatively
    rare. If the variance lower bound falls off quickly as the given
    construction is perturbed, then the associated distribution may still be
    approximately self-normalizable with a good rate.
\end{enumerate}

We have already seen that new theoretical insights in this domain can translate
directly into practical applications. Thus, in addition to their inherent
theoretical interest, answers to each of these questions might be applied
directly to the training of approximately self-normalized models in practice. We
expect that self-normalization will find increasingly many applications, and we
hope the results in this paper provide a first step toward a complete
theoretical and empirical understanding of self-normalization in log-linear models.

\newpage

\bibliography{jacob}
\bibliographystyle{plain}

\onecolumn

\appendix

\section{Normalizable distributions}
\label{app:norm-dist}

\begin{proof}[Proof of \autoref{prop:approx-closeness} (distributions close to
  normalizable sets are approximately normalizable)]
  \strut\\[1em]
  Let $T(x,y) = T^*(x,y) + T^-(x,y)$, where 
  $\displaystyle T^*(x,y) = \argmin_{T(x,y) : x \in \normset} ||T(X,y) - T(x,y)||_2$ .

  Then,
  \begin{align*}
    \expect\left(\log\left( \int \cexp{\eta^\top T(X,y)} \ld y \right) \right)^2
    &= \expect\left(\log\left( \int \cexp{\eta^\top (T^*(X,y) + T^-(X,y)}) \ld y
    \right) \right)^2 \\
    &\leq \expect\left(\log\left( \cexp{\eta^\top \tilde{T}} \int
    \cexp{\eta^\top T^*(X,y)} \ld y \right) \right)^2 \\
    \intertext{for $\displaystyle \tilde{T} = \argmax_{T(X,y)} ||\eta^\top T(X,y)||_2$,}
    % TODO clarify sup taken over y
    &\leq \expect\left(\log\left( \cexp{\eta^\top \tilde{T}} \right)\right)^2 \\
    &= (DB)^2 \qedhere
  \end{align*}
\end{proof}

\section{Normalization and likelihood}
\label{app:lgap}

\subsection{General bound}

% TODO also bookkeeping involving mean

\begin{lem}
  \label{lem:shrinkage}
  If $||\eta||_2 \leq \delta / R$, then $\prob{y}{x}{\eta}$ is
  $\delta$-approximately normalized about $\log \mu(\ys)$.
\end{lem}

\begin{proof}
  If $\int \cexp{\eta^\top T(X,y)} \ld \mu(y) \geq \log \mu(\ys)$,
  \begin{align*}
    \left(\log \int_\ys \cexp{\eta^\top T(X,y)} \ld \mu(y) - \log \mu(\ys)\right)^2
    &\leq \left(\log \int_\ys \cexp{||\eta||_2 R} \ld \mu(y) - \log
    \mu(\ys)\right)^2 \\
    &= ||\eta||^2_2 R^2 \\
    &\leq \delta^2
  \end{align*}
  The case where $\int \cexp{\eta^\top T(X,y)} \ld \mu(y) \leq \log \mu(\ys)$ is
  analogous, instead replacing $\eta^\top T(x,y)$ with $-||\eta||_2 R$. The
  variance result follows from the fact that every log-partition is within
  $\delta$ of the mean.
\end{proof}

\begin{proof}[Proof of \autoref{prop:lgap} (loss of likelihood is bounded in
  terms of distance from uniform)]
  Consider the likelihood evaluated at $\alpha\mle$, where $\alpha = \delta /
  R||\mle||_2$. We know that $0 \leq \alpha \leq 1$ (if $\delta > R\eta$, then
  the MLE already satisfying the normalizing constraint). Additionally,
  $\prob{y}{x}{\alpha\mle}$ is $\delta$-approximately normalized. (Both follow
  from \autoref{lem:shrinkage}.)

  Then,
  \begin{align*}
    \Delta_\ell &= \frac{1}{n} \sum_i \left[ (\mle^\top T(x_i, y_i) - A(x_i, \mle)) - 
                                 (\alpha\mle^\top T(x_i, y_i) - A(x_i,
                                 \alpha\mle)) \right] \\
                &= \frac{1}{n} \sum_i \left[
                                 (1-\alpha)\mle^\top T(x_i, y_i)
                                 - A(x_i, \mle) + A(x_i, \alpha\mle)
                                \right] \\
    \intertext{Because $A(x, \alpha \eta)$ is convex in $\alpha$,}
    A(x_i, \alpha\mle) &\leq (1-\alpha)A(x_i, \mathbf{0}) + \alpha A(x_i, \mle)
    \\
    &= (1-\alpha)\mu(\ys) + \alpha A(x_i, \mle)
    \intertext{Thus,}
    \Delta_\ell &= \frac{1}{n} \sum_i \left[
                                 (1-\alpha)\mle^\top T(x_i, y_i)
                                 - A(x_i, \mle) + (1-\alpha)\log \mu(\ys) + \alpha
                                 A(x_i, \mle)
                                \right] \\
                &= (1-\alpha) \frac{1}{n} \sum_i \left[
                                 \mle^\top T(x_i, y_i)
                                 - A(x_i, \mle) + \log \mu(\ys)
                                \right] \\
                &= (1-\alpha) \frac{1}{n} \sum_i \left[
                  \log \prob{y}{x}{\eta} - \log \mathrm{Unif}(y)
                              \right] \\
    &\asymp (1-\alpha)
    ~
    \expect\, \mathrm{KL}(\prob{\cdot}{X}{\eta}\ ||\ \mathrm{Unif}) \\
    &\leq \left(1 - \frac{\delta}{R||\mle||_2}\right)
    ~
    \expect\, \mathrm{KL}(\prob{\cdot}{X}{\eta}\ ||\ \mathrm{Unif})
    \qedhere
  \end{align*}
\end{proof}

\subsection{All-nonuniform bound}

We make the following assumptions:
\begin{itemize}
  \item
    Labels $y$ are discrete. That is, $\ys = \{1, 2, \dots, k\}$ for some $k$.
  \item
    $x \in \hypercube(d)$. That is, each $x$ is a $\{0,1\}$ indicator vector drawn
    from the Boolean hypercube in $q$ dimensions.
  \item
    Joint feature vectors $T(x,y)$ are just the features of $x$ conjoined with
    the label $y$. Then it is possible to think of $\eta$ as a sequence of
    vectors, one per class, and we can write $\eta^\top T(x,y) = \eta_y^\top x$.

  \item
    As in the body text, let all MLE predictions be nonuniform, and in
    particular let each $\mle_{y^*}^\top x - \mle_y^\top x > c||\mle||$ for $y
    \neq y^*$. 
\end{itemize}

\begin{lem}
  \label{prop:one-cov}
  For a fixed $x$, the maximum covariance between any two features $x_i$
  and $x_j$ under the model evaluated at some $\eta$ in the direction of the MLE:
  \begin{equation}
    \cov[T(X,Y)_i, T(X,Y)_j|X=x] \leq 2(k-1)\cexp{-c\delta}
  \end{equation}
  %(This is the only hard case.  As noted earlier, each $T(x,y)_{y',i}$ is
  %identically zero for $y \neq y'$, and so has zero covariance with every other
  %feature. Additionally, $T(x,y)_{y,i}$ is zero if $x_i = 0$.)
\end{lem}

\begin{proof}
  If either $i$ or $j$ is not associated with the class $y$, or associated with
  a zero element of $x$, then the associated feature (and thus the covariance at
  $(i,j)$) is identically zero.  Thus we assume that $i$ and $j$ are both
  associated with $y$ and correspond to nonzero elements of $x$.

  \begin{align*}
    \cov[T_i,T_j|X=x] &= \sum_y p_\eta(y|x) - p_\eta(y|x)^2 \\
    \intertext{Suppose $y$ is the majority class. Then,}
    p_\eta(y|x) - p_\eta(y|x)^2 &= \frac{\cexp{\eta_y^\top x}}{\sum_{y'}
    \cexp{\eta_{y'}^\top x}} - \frac{\cexp{2\eta_y^\top x}}{\left(\sum_{y'}
    \cexp{\eta_{y'}^\top x}\right)^2} \\
    &= \frac{\cexp{\eta_y^\top x} \left( \sum_{y'}\cexp{\eta_{y'}^\top x} \right)
    - \cexp{2 \eta_y^\top x}}{\left(\sum_{y'} \cexp{\eta_{y'}^\top x}\right)^2}
    \\
    &\leq \frac{\cexp{\eta_y^\top x} \left( \sum_{y'}\cexp{\eta_{y'}^\top x} \right)
    - \cexp{2 \eta_y^\top x}}{\cexp{2\eta_{y}^\top x}} \\
    &= \sum_{y' \neq y} \cexp{(\eta_y' - \eta_y)^\top x} \\
    &\leq (k-1)\cexp{-c||\eta||} \\
    \intertext{Now suppose $y$ is not in the majority class. Then,}
    p_\eta(y|x) - p_\eta(y|x)^2 &\leq p(y|x) \\
    &= \frac{\cexp{\eta_y^\top x}}{\sum_{y'} \cexp{\eta_{y'}^\top x}} \\
    &\leq \cexp{-c||\eta||}
    \intertext{Thus the covariance}
    \sum_y p_\eta(y|x) - p_\eta(y|x)^2 &\leq 2(k-1)\cexp{-c||\eta|||}
  \end{align*}
\end{proof}

\begin{lem}
  Suppose $\eta = \beta\mle$ for some $\beta < 1$. Then for a
  sequence of observations $(x_1, \dots, x_n)$, under the model evaluated at
  $\xi$, the largest eigenvalue of the feature covariance matrix
  \begin{equation}
    \label{eq:full-cov}
    \frac{1}{n} \sum_i \left[ \expect_\xi[T T^\top|X = x_i] -
    (\expect_\theta[T|X = x_i])(\expect_\xi[T|X = x_i])^\top \right]
  \end{equation}
  is at most
  \begin{equation}
    q (k-1) e^{-c\beta||\mle||}
  \end{equation}
\end{lem}

\begin{proof}
  %%We would like to apply the preceding proposition immediately, but we have no
  %%way of ensuring that for a given $\theta$ and $x$ there is a unique class
  %%assigned highest probability under the model. However, given a large enough
  %%collection of unique $x_i$ we can bound the number of non-uniform predictions
  %%that can occur: in particular, no more than the number of solutions to the
  %%system of equations
  %%\begin{equation}
  %%  \left[ \begin{array}{c} T(x_1,y_1)^\top \\ \vdots \\ T(x,y_k) \\ T(x_2,y_1) \\
  %%    \vdots \\ T(x_n, y_k)^\top \end{array} \right]
  %%  \left[ \begin{array}{c} \theta_1 \\ \vdots \\ \theta_{kd} \end{array} \right]
  %%  =
  %%  \left[ \begin{array}{c} a_1 \\ \vdots \\ a_1 \\ a_2 \\ \vdots \\ a_{n} \end{array} \right]
  %%\end{equation}
  %%in unknowns $\theta_i$ and $a_i$. This is effectively $nk$ equations in
  %%$kd + n$ unknowns, so we conclude that the system is underdetermined
  %%(i.e.\ uniform predictions can be made for all $\theta$) only when $n < kd /
  %%(k-1)$. In fact this bound is quite pessimistic, and for a fixed set of $x$s,
  %%almost every $\theta$ gives only nonuniform predictions. Call the number of
  %%uniform predictions $c_2$. Then any individual entry in the a covariance
  %%matrix for a uniform prediction is at most 1, and for a nonuniform prediction
  %%at most $e^{-c_1\sqrt{d}\delta}$. 
  From \autoref{prop:one-cov}, each entry in the covariance matrix is at most
  $(k-1)\cexp{-c||\eta||} = (k-1)\cexp{-c\beta||\mle||}$. At most $q$ features
  are nonzero active in any row of the matrix.
  Thus by Gershgorin's theorem, the maximum eigenvalue
  of each term in \autoref{eq:full-cov} is $q(k-1)\cexp{-c\beta||\mle||}$, which
  is also an upper bound on the sum.
\end{proof}

\begin{proof}[Proof of \autoref{prop:strong-lgap} (loss of likelihood goes as
  $\cexp{-\delta}$)]
  As before, let us choose $\dmle = \alpha\mle$, with $\alpha =
  \delta/R||\mle||_2$. We have already seen that this choice of parameter is
  normalizing.

  Taking a second-order Taylor expansion about $\eta$, we have
  \begin{align*}
    \log p_{\dmle}(y|x) &= \log p_\eta(y|x) + (\dmle - \mle)^\top
    \grad \log p_\mle(y|x)
    + (\dmle - \mle)^\top \grad\grad^\top \log p_\xi(y|x) (\dmle -
    \mle) \\
   &= \log p_\mle(y|x) + (\dmle - \mle)^\top \grad\grad^\top \log
   p_\xi(y|x) (\dmle - \mle) \\
   \intertext{where the first-order term vanishes because $\mle$ is the MLE.
     It is a standard result for exponential families that the Hessian in the
     second-order term is just \autoref{eq:full-cov}.
     %It must
     %be the case that for each $\xi$ the normalization penalty is greater than
     %$\delta$ (likelihood increases monotonically from $\theta_\delta$ to
     %$\theta$, so if there is a $\xi$ within normalization constraints then
     %$\theta_\delta$ is not the constrained MLE). 
     Thus we can write} 
     &\geq \log p_\mle(y|x) - ||\dmle - \mle||^2 q(k-1)\cexp{-c\beta||\eta||}
     \\
     &\geq \log p_\mle(y|x) - (1-\alpha)^2||\mle||^2
     q(k-1)\cexp{-c\alpha||\eta||} \\
     &= \log p_\mle(y|x) - (||\mle|| - \delta/R)^2 q(k-1)\cexp{-c\delta/R}
  \end{align*}
  The proposition follows.
\end{proof}

\section{Variance lower bound}
\label{app:var-lower-bd}

Let
$$ U_0 = \lbrace \beta \in \R^{Kd} \colon \exists \tilde{\beta} \in \R^{d}, \beta_{kj} = \tilde{\beta}_{j},~1 \leq k \leq K,~ 1 \leq j \leq d \rbrace . $$

\begin{lem}\label{lem:eta-equiv-lin}
If $\lspan\left(\xs\right) = \R^{d}$, then equivalence of natural parameters is characterized by
$$ \eta \sim \eta' \Longleftrightarrow \eta - \eta' \in U_0 . $$
\end{lem}
\begin{proof}
For $x \in \xs$, denote by $P_{\eta}(x) \in \Delta_{K}$ the distribution over $\ys$. Now, suppose that $\eta \sim \eta'$ and fix $x \in \xs$. By the definition of equivalence, we have
$$ \frac{P_{\eta}(x)_{k}}{P_{\eta}(x)_{k'}} = \frac{P_{\eta'}(x)_{k}}{P_{\eta'}(x)_{k'}} , $$
which immediately implies
$$ \left(\eta_{k} - \eta_{k'}\right)^{T}x = \left(\eta'_{k} - \eta'_{k'}\right)^{T} x  , $$
whence
$$ \left[\left(\eta_{k} - \eta'_{k}\right) - \left(\eta_{k'} - \eta'_{k'}\right)\right]^{T}x = 0 . $$
Since this holds for all $x \in \xs$ and $\lspan(\xs) = \R^{d}$, we get
$$ \eta_{k} - \eta'_{k} = \eta_{k'} - \eta'_{k'} . $$
That is, if we define
$$ \tilde{\beta}_{j} = \eta_{1j} - \eta'_{1j}, $$
we get
$$ \eta_{kj} - \eta'_{kj} = \eta_{1d} - \eta'_{1j} = \tilde{\beta}_{j}, $$
and $\eta - \eta' \in U_0$, as required.

Conversely, if $\eta - \eta' \in U_0$, choose an appropriate $\tilde{\beta}$. We then get
$$ \eta^{T}_{k}x = \left(\eta'\right)^{T}x + \tilde{\beta}^{T}x . $$
It follows that
$$ A(\eta', x) = A(\eta,~x) + \tilde{\beta}^{T}x, $$
so that
$$ \eta^{T}T(k,~x) - A(\eta,~x) = \left(\eta'\right)^{T}x + \tilde{\beta}^{T}x - \left[A(\eta',~x) + \tilde{\beta}^{T}x\right] = \left(\eta'\right)^{T}x - A(\eta',~x) $$
and the claim follows.
\end{proof}

The key tool we use to prove the theorem reinterprets $V^{\ast}(\eta)$ as the norm of an orthogonal projection. We believe this may be of independent interest. To set it up, let $\calS = \LQR$ be the Hilbert space of square-integrable functions with respect to the
input distribution $p(x)$, define
$$ w_{j}(x) = x_j - \E_{p(x)}\left[X_j\right] $$
and
$$ \calC =\lspan\left(w_{j}\right)_{1 \leq j \leq d} . $$
We then have

\begin{lem}\label{lem:Vast-proj}
Let $\tilde{A}(\eta,~x) = A(\eta,~x) - \E_{p(x)}\left[A(\eta,~X)\right]$. Then
$$ V^{\ast}(\eta) = \nml \tilde{A}(\eta,~ \cdot) - \projC \tilde{A}(\eta,~ \cdot) \nmr_{2}^{2} .  $$
\end{lem}

The second key observation, which we again believe is of independent interest,
is that under certain circumstances, we can completely replace the normalizer
$A(\eta,~\cdot)$ by $\max_{y \in \ys}{\eta^{T}T(y,~x)}$. For this, we define
$$ E_{\infty}(\eta)(x) = \max_{k} \eta^{T}T(k,~x) = \max_{k} \eta^{T}_{k}x $$
and correspondingly let $\bar{E}_{\infty}(\eta) = \E_{p(x)}\left[E_{\infty}(\eta)(x)\right]$. 
\begin{proof}%[Proof of Lemma \ref{lem:Vast-proj}]
By Lemma \ref{lem:eta-equiv-lin}, we have
$$ V^{\ast}(\eta) = \inf_{\beta \in \R^{d}} \int_{\R^{Kd}}{\left[A(\eta,~x) - \bar{A}(\eta) - \left(\beta^{T}x - \beta^{T}\E_{p(x)}\left[X\right]\right)\right]^{2} \der p(x)} . $$
But now, we observe that this can be rewritten with the aid of the isomorphism $\R^{d} \simeq \calC$ defined by the identity
$$ \beta^{T}x - \beta^{T}\E_{p(x)}\left[X\right] = \sum_{j} \beta_{j}w_{j}(x) $$
to read
$$ V^{\ast}(\eta) = \inf_{f \in \calC}\int_{\R^{d}}{\left[A(\eta,~x) - \bar{A}(\eta) - f\right]^{2} \der p(x)} = \nml \tilde{A}(\eta,~\cdot) - \projC \tilde{A}(\eta,~ \cdot)\nmr_{2}^{2} , $$ 
as required.
\end{proof}

\begin{lem} \label{lem:AEinf-bd}
Suppose for each $x \in \xs$, there is a unique $k^{\ast} = k^{\ast}(x)$ such that $k^{\ast}(x) = \argmax_{k} \eta^{T}_{k}x$ and such that for
$k \neq k^{\ast}$, $\eta^{T}_{k}x \leq \eta^{T}_{k^{\ast}}x - \Delta$ for some $\Delta > 0$. Then
$$ \sup_{x \in \xs} \left|A(\eta,~x) - \bar{A}(\eta) - \left[E_{\infty}(\eta)(x) - \bar{E}_{\infty}(\eta)\right]\right| \leq Ke^{-\Delta\alpha} . $$
\end{lem}
\begin{proof}%[Proof of Lemma \ref{lem:AEinf-bd}]
Denote by $\tilde{E}_{\infty}$ the centered version of $E_{\infty}$. Using the identity $1 + t \leq e^{t}$, we immediately see that
$$ E_{\infty}(\alpha\eta)(x) \leq A(\alpha\eta,~x) = \alpha E_{\infty}(\eta)(x) + \log\left(1 + \sum_{k \neq k^{\ast}(x)} e^{\left[\eta_{k}^{T}x - E_{\infty}(\eta)(x)\right]}\right) \leq E_{\infty}(\alpha\eta)(x) + Ke^{-\Delta\alpha} . $$
It follows that
$$ \E_{p(x)}\left[E_{\infty}(\alpha\eta)(X)\right] \leq \E_{p(x)}\left[A(\alpha\eta,~X)\right] \leq \E_{p(x)}\left[E_{\infty}(\alpha\eta)(X)\right] + Ke^{-\Delta\alpha} . $$
We thus have
$$ -Ke^{-\Delta\alpha} \leq \tilde{A}(\alpha\eta,~x) - \tilde{E}_{\infty}(\alpha\eta)(x) \leq Ke^{-\Delta\alpha},~~ x \in \xs . $$
The claim follows.
\end{proof}

If we let
$$ \VE^{\ast}(\eta) = \inf_{\eta' \sim \eta} \Var_{p(x)}\left[\tilde{E}_{\infty}(\eta',~X)\right] . $$

\begin{cor}\label{cor:VAst-VE-bd}
For $\alpha > \frac{\log{2K}}{\Delta}$, we have
$$ V^{\ast}(\alpha\eta) \geq \VE^{\ast}(\eta)\alpha^{2} - \left(1 + \VE^{\ast}(\eta)\right)\alpha . $$
\end{cor}
\begin{proof}
For this, observe first that if $\eta' \sim \eta$, then
\begin{align*}
\tilde{A}(\eta',~x)^{2} & \geq \tilde{E}_{\infty}(\alpha\eta')(x)^{2} - 2\left|\tilde{E}_{\infty}(\alpha\eta')(x)\right|\left|\tilde{A}(\eta',~x) - \tilde{E}_{\infty}(\eta')(x)\right| . 
\end{align*}
By linearity of $E_{\infty}(\eta')$ in its $\eta$ argument, and by Lemma \ref{lem:AEinf-bd}, we therefore deduce
$$ \tilde{A}(\eta',~x)^{2} \geq \tilde{E}_{\infty}(\eta')(x)^{2}\alpha^{2} - 2Ke^{-\Delta\alpha}\left|\tilde{E}_{\infty}(\eta')(x)\right|\alpha . $$
Then using the inequality $\E_{p(x)}\left[\left|f(X)\right|\right] \leq 1 + \Var_{p(x)}\left[f(X)\right]$, valid for any $f \in \LQR$ with $\E_{p(x)}\left[f\right] = 0$, we thus deduce
$$ \Var_{p(x)}\left[A(\alpha\eta',~X)\right] \geq \Var_{p(x)}\left[E_{\infty}(\eta')(X)\right]\alpha^{2} - 2Ke^{-\Delta\alpha}\left(1 + \Var_{p(x)}\left[E_{\infty}(\eta')(X)\right]\right)\alpha . $$
Taking the infimum over both sides, we get
$$ V^{\ast}(\eta) \geq \VE^{\ast}(\eta) - 2Ke^{-\Delta\alpha}\left(1 + \VE^{\ast}(\eta)\right)\alpha . $$
\end{proof}

We are now prepared to give the explicit example. It is defined by $\eta_{k} = 0$ if $k > 2$ and
\begin{equation}\label{eq:eta-hard-1}
\eta_{1j} = \begin{cases}
-a &~ \text{if}~ d = 1, \\
\frac{a}{d - 1} &~ \text{o.w.}
\end{cases} 
\end{equation}
and for all $j$, 
\begin{equation}\label{eq:eta-hard-2}
\eta_{2j} = \frac{a}{d(d - 1)} , 
\end{equation}
where
$$ a = \sqrt{1 - \frac{1}{d}} . $$
For convenience, also define
$$ b(x) = \sum_d x_d $$
and observe that
$$ E_{\infty}(\eta)(x) = \begin{cases} 
\frac{ab(x)}{d\left(d - 1\right)} &~ \text{if}~ x_{j} = 1, \\
\frac{ab(x)}{d - 1} &~ \text{o.w.}
\end{cases}, $$

Our goal will be to prove that
$$ 1 \geq \VE^{\ast}(\eta) \geq \frac{1}{32d(d - 1)} . $$
The claim will then follow by the above corollary. 

To see that $\VE^{\ast}(\eta) \leq 1$, we simply note that
$$ \max_{k} \left|\eta^{T}_{k}x \right| \leq a < 1, $$
whence $\Var_{p(x)}\left[\eta^{T}x\right] \leq 1$ as well and we are done.

The other direction requires more work. To prove it, we first prove the following lemma
\begin{lem}\label{lem:Einf2-bd} With $\eta$ defined as in \eqref{eq:eta-hard-1}-\eqref{eq:eta-hard-2}, we have
$$ \inf_{\eta' \sim \eta} \E_{p(x)}\left[E_{\infty}(\eta')(X)^{2}\right] \geq \frac{1}{16d(d - 1)} . $$
\end{lem}
\begin{proof}
Suppose $\eta_{k} - \eta_{k}' = \beta \in \R^{d}$. We can then write
$$ \inf_{\eta' \sim \eta} \E_{p(x)}\left[E_{\infty}(\eta')(X)^{2}\right] = \inf_{\beta \in \R^{d}} \frac{1}{2^d}\sum_{x \in \hypc}\sum_{x \in \hypc}{\left[E_{\infty}(\eta)(x) - \beta^{T}x\right]^{2}} $$
and we therefore define
\begin{align*}
\mathcal{L}(\beta) & = \sum_{x \in \hypc}\sum_{x \in \hypc}{\left[E_{\infty}(\eta)(x) - \beta^{T}x\right]^{2}} \\
                            & = \sum_{x \colon x_1 = 0}\left[\left(\beta_1 + \beta^{T}x - \frac{a}{d\left(d - 1\right)}\right)^{2} + \left(\frac{ab(x)}{d - 1} - \beta^{T}x\right)^{2} \right] , \\
\end{align*}
noting that
$$ \inf \mathcal{L} = 2^{d} \cdot \inf_{\eta' \sim \eta} \E_{p(x)}\left[E_{\infty}(\eta')(X)^{2}\right] . $$
We therefore need to prove
$$ \mathcal{L} \geq \frac{2^{d - 4}}{d(d - 1)} . $$

Holding $\beta_{2:d}$ fixed, we note that the optimal setting of $\beta_1$ is given by
$$ \beta_1 = -\frac{1}{2}\sum_{j \geq 2}{\beta_j} + \frac{a}{d\left(d - 1\right)} . $$
We can therefore work with the objective
$$ \mathcal{L}(\beta) =  \sum_{x \colon x_1 = 0}{\left[\frac{\left(\beta^{T}x - \beta^{T}x^{\neg}\right)^{2}}{4} + \left(\frac{ab(x)}{d - 1} - \beta^{T}x\right)^{2}\right]} , $$
where we have defined
$$ x^{\neg}_{j} = \begin{cases}
0 &~ \text{if}~ j = 1, \\
1 - x_{j} &~ \text{o.w.}
\end{cases} $$

Grouping into $\lbrace x,~ x^{\neg} \rbrace$ pairs, we end up with
$$ \mathcal{L}(\beta_{2:d}) = \sum_{x \colon x_1 = x_2 = 0}{\left[\frac{\left(\beta^{T}x - \beta^{T}x^{\neg}\right)^{2}}{2} + \left(\frac{ab(x)}{d - 1} - \beta^{T}x\right)^{2} + \left(\frac{ab(x^{\neg})}{d - 1} - \beta^{T}x^{\neg}\right)^{2} \right]} $$
Now, supposing $b(x) \leq \frac{d - 1}{2} - \frac{3}{2}$ or $b(x) \geq \frac{D - 1}{2} + \frac{3}{2}$, we have
$$ \left|b(x^{\neg}) - b(x)\right| = \left|d - 1 - 2b(x)\right| \geq 3 . $$
We will bound the terms that satisfy this property. Indeed, supposing we fix such an $x$, at least one of the following must be true: either 
$$ \max\left(\left(\frac{ab(x)}{d - 1} - \beta^{T}x\right)^{2},~~ \left(\frac{ab(x^{\neg})}{d - 1} - \beta^{T}x^{\neg}\right)^{2}\right) \geq \frac{a^{2}}{\left(d - 1\right)^{2}} , $$
or
$$ \left(\beta^{T}x - \beta^{T}x^{\neg}\right)^{2} \geq \frac{a^{2}}{\left(d - 1\right)^{2}} . $$
Indeed, suppose the first condition does not hold. Then necessarily
$$ \left|\frac{ab(x)}{d - 1} - \beta^{T}x \right| < \frac{a}{d - 1} $$
and
$$ \left|\frac{ab(x^{\neg})}{d - 1} - \beta^{T}x^{\neg}\right| < \frac{a}{d - 1} , $$
so that
$$ \frac{a\left(b(x) - 1\right)}{d - 1} \leq \beta^{T}x \leq \frac{a\left(b(x) + 1\right)}{d - 1} $$
and
$$ \frac{a\left(b(x^{\neg}) - 1\right)}{d - 1} \leq \beta^{T}x \leq \frac{a\left(b(x^{\neg}) + 1\right)}{d - 1} . $$
Now, if $b(x) \geq b(x^{\neg}) + 3$, this immediately implies
$$ \beta^{T}x - \beta^{T}x^{\neg} \geq \frac{a}{d - 1} $$
and, symmetrically, if $b(x^{\neg}) \geq b(x) + 3$, we get
$$ \beta^{T}x^{\neg} - \beta^{T}x \geq \frac{a}{d - 1} . $$
Either way, the second inequality holds, whence the claim. Since there are at least $2^{d - 1} - \frac{3 \cdot 2^{d}}{\sqrt{\frac{3d}{2} + 1}} \geq 2^{d - 2}$ choices of $x$ satisfying the requirements of our line of reasoning, we get $2^{d - 3}$ pairs, whence
$$ \mathcal{L}(\beta_{2:d}) \geq \frac{2^{d - 4}a^{2}}{\left(d - 1\right)^{2}} = \frac{2^{d - 4}}{d\left(d - 1\right)} , $$
as claimed. 
\end{proof}

We can apply this lemma to derive a variance bound, viz.

\begin{lem}\label{lem:VE-lobd} With $\eta$ as in \eqref{eq:eta-hard-1}-\eqref{eq:eta-hard-2}, we have
$$ \VE^{\ast}(\eta) \geq \frac{1}{32d(d - 1)} . $$
\end{lem}
\begin{proof}
For this, observe that, with $\eta'$ being the value corresponding to $\eta_{k}' - \eta_{k} = \beta$, we have
\begin{align*}
\VE^{\ast}(\eta) = \inf_{\beta} \frac{1}{2^d}\sum_{x \in \hypc}{\tilde{E}_{\infty}(\eta')(x)^{2}} \geq \inf_{\beta} \frac{1}{2^{d}} \sum_{x \in \hypc \colon x_1 = 1}{\tilde{E}_{\infty}(\eta')(x)^{2}} .
\end{align*}
Applying the previous result to the $(D - 1)$-dimensional hypercube on which $x_1 = 1$, we deduce
$$ \VE^{\ast}(\eta) \geq \frac{1}{2} \cdot \frac{1}{16(d - 1)(d - 2)} = \frac{1}{32(d - 1)(d - 2)} \geq \frac{1}{32d(d - 1)} . $$
\end{proof}

\begin{proof}[Proof of Theorem \ref{thm:norm-var-lo} from Lemma \ref{lem:VE-lobd}]
Putting everything together, we see first that
$$ V^{\ast}(\alpha\eta) \geq \VE^{\ast}(\eta)\alpha^{2} - 4e^{-\Delta\alpha}\alpha , $$
where $\Delta = \frac{\sqrt{1 - \frac{1}{d}}}{2(d - 1)}$. But then this implies
$$ V^{\ast}(\alpha\eta) \geq \frac{\alpha^{2}}{32d(d - 1)} - 4e^{-\Delta\alpha}\alpha . $$
On the other hand, $\nml \eta \nmr_{2}^{2} \leq 2$, so $\alpha^{2} = \frac{\nml \alpha\eta \nmr_{2}^{2}}{\nml \eta \nmr_{2}^{2}} \geq \frac{\nml \alpha\eta \nmr_{2}^{2}}{2}, $ whence
$$ V^{\ast}(\alpha\eta) \geq \frac{\left|\left|\alpha\eta\right|\right|_{2}^{2}}{64d(d - 1)} - 4e^{-\frac{\sqrt{1 - \frac{1}{d}}\nml \alpha\eta \nmr_{2}}{2(d - 1)}}\nml \alpha \eta \nmr_{2}, $$
which is the desired result.  
\end{proof}

\end{document}